\documentclass{article}

\usepackage{microtype}
\usepackage{graphicx}
\usepackage{subfigure}
\usepackage{booktabs} 
\usepackage{tcolorbox}
\usepackage{hyperref}

\usepackage{multirow}

\usepackage[accepted]{icml2024}

\usepackage{amsmath}
\usepackage{amssymb}
\usepackage{mathtools}
\usepackage{amsthm}

\usepackage[capitalize,noabbrev]{cleveref}

\theoremstyle{plain}
\newtheorem{theorem}{Theorem}[section]

\newtheorem{lemma}[theorem]{Lemma}

\theoremstyle{definition}

\theoremstyle{remark}

\usepackage[textsize=tiny]{todonotes}

\icmltitlerunning{DiJiang: Efficient Large Language Models through Compact Kernelization}

\begin{document}

\twocolumn[
\icmltitle{DiJiang: Efficient Large Language Models through Compact Kernelization}



\icmlsetsymbol{equal}{*}

\begin{icmlauthorlist}
\icmlauthor{Hanting Chen}{equal,yyy}
\icmlauthor{Zhicheng Liu}{equal,yyy}
\icmlauthor{Xutao Wang}{yyy}
\icmlauthor{Yuchuan Tian}{pku}
\icmlauthor{Yunhe Wang}{yyy}
\end{icmlauthorlist}

\begin{center}
	\{chenhanting,yunhe.wang\}@huawei.com;
\end{center}

\icmlaffiliation{yyy}{Huawei Noah’s Ark Lab}
\icmlaffiliation{pku}{Peking University}

\icmlcorrespondingauthor{Yunhe Wang}{yunhe.wang@huawei.com}

\icmlkeywords{Machine Learning, ICML}

\vskip 0.3in
]



\printAffiliationsAndNotice{\icmlEqualContribution} 

\begin{abstract}
In an effort to reduce the computational load of Transformers, research on linear attention has gained significant momentum. However, the improvement strategies for attention mechanisms typically necessitate extensive retraining, which is impractical for large language models with a vast array of parameters. In this paper, we present DiJiang, a novel Frequency Domain Kernelization approach that enables the transformation of a pre-trained vanilla Transformer into a linear complexity model with little training costs. By employing a weighted Quasi-Monte Carlo method for sampling, the proposed approach theoretically offers superior approximation efficiency. To further reduce the training computational complexity, our kernelization is based on Discrete Cosine Transform (DCT) operations. Extensive experiments demonstrate that the proposed method achieves comparable performance to the original Transformer, but with significantly reduced training costs and much faster inference speeds. Our DiJiang-7B achieves comparable performance with LLaMA2-7B on various benchmark while requires only about 1/50 training cost. Code is available at \url{https://github.com/YuchuanTian/DiJiang}.
\end{abstract}

\section{Introduction}
The Transformer architecture~\cite{vaswani2017attention} has revolutionized the field of Natural Language Processing (NLP), achieving outstanding results in various tasks such as speech recognition~\cite{dong2018speech}, machine translation~\cite{wang2019learning}, and document generation/summarization~\cite{kim2022ocr}. This success has led to an era dominated by large language models (LLMs), where the Transformer structure is scaled up to handle increasingly complex tasks. However, this scaling brings with it substantial computational demands, especially due to the attention mechanism which requires cross-correlation calculations between each token. These computational requirements, coupled with the significant inference costs and energy consumption, present considerable obstacles to deploying these models in resource-constrained environments like mobile devices and robotics.

In response to the pressing need for more efficient Transformer models, the research community has directed its efforts towards optimizing the Transformer architecture. A myriad of strategies has been put forward, encompassing methods such as model pruning, quantization, and the development of more efficient attention mechanisms. Among these initiatives, simplifying the attention mechanism has emerged as a particularly promising avenue. This approach focuses on transforming the traditionally quadratic complexity of attention mechanisms into a more manageable linear scale. ~\cite{katharopoulos2020transformers} introduces Linear Transformers, which leverage kernel feature maps to transform self-attention, reducing complexity from quadratic to linear while maintaining comparable results to traditional Transformers.~\cite{kitaev2020reformer} proposes replacies dot-product attention with locality-sensitive hashing and using reversible residual layers to minimize memory usage in training. Performer~\cite{choromanski2020rethinking} utilize positive orthogonal random features to approximate softmax-based self-attention in Transformers, achieving a transformative leap to linear complexity. 

However, the majority of existing methods for optimizing Transformers, particularly in relation to their attention mechanisms, necessitate comprehensive retraining. This retraining process presents a formidable challenge, especially for models with an immense array of parameters. It requires a significant investment in terms of computational resources and time. For instance, the training of a large model like LLaMA-7B~\cite{touvron2023llama} demands approximately 82,432 GPU-hours and incurs a total power consumption of around 36 MWh. Undertaking such extensive retraining for models of this magnitude is not only economically taxing but also raises environmental concerns due to the substantial energy expenditure involved. This underscores the need for more efficient approaches to adapt and optimize these large-scale models. Undertaking such extensive retraining for models of this magnitude is not only economically taxing but also raises environmental concerns due to the substantial energy expenditure involved. Despite few research~\cite{zheng2023efficient,choromanski2020rethinking} efforts focusing on finding fast approximations for attention mechanisms, these methods have not been thoroughly validated in large-scale language models.

To address the issue of fast attention approximations in large language models, we conducted a thorough analysis of existing linear attention schemes. We discovered that the main source of approximation error in these methods is due to sampling based on the Monte Carlo method. Consequently, we propose the use of weighted Quasi-Monte Carlo sampling for mapping, specifically introducing Frequency Domain Kernelization. This approach efficiently and accurately maps the queries and keys of a Transformer to the frequency domain using Discrete Cosine Transform (DCT). This mapping allows us to effectively eliminate the softmax operation in the attention mechanism, rendering the attention computation linear in complexity, which is shown in Figure~\ref{fig:method}. We theoretically demonstrate that this frequency domain mapping is an approximate equivalent to the original attention mechanism. Our experiments show that our method achieves performance comparable to the original Transformer with a significantly smaller training cost ($<1/10$), while also benefiting from faster inference speeds (up to about 10x). 

\section{Related Works}

\subsection{Linear Transformers}

Reducing the computational load of attention in Transformers remains a hot topic in research.~\cite{child2019generating} achieved this by sparsifying attention, thereby reducing its computational cost. Similarly,~\cite{kitaev2020reformer} used locality-sensitive hashing to expedite the computation of attention. However, these methods are hard to apply in auto-regressive Transformer models. As a result, there has been a series of works focusing on removing or substituting the softmax in attention. Notably, the Linear Transformer, first introduced by~\cite{katharopoulos2020transformers}, represents a significant stride in this direction.~\cite{qin2022cosformer} approximated attention calculations using a linear operator and a cosine-based distance reweighting.~\cite{zhai2021attention} achieved linear complexity in Transformers by preprocessing keys and values.~\cite{lu2021soft} used Gaussian kernel functions in place of dot-product similarity, allowing for the approximation of the full self-attention matrix through low-rank matrix decomposition.~\cite{bello2021lambdanetworks} bypassed the need for attention calculations by capturing interactions through transforming available contexts into linear functions and applying them to each input, showcasing the variety of methods explored to optimize attention mechanisms in Transformer models.

Additionally, recent proposals like RWKV~\cite{peng2023rwkv}, RetNet~\cite{sun2023retentive}, and Mamba~\cite{gu2023mamba} have introduced potential alternatives to the Transformer with linear complexity. However, these existing improvements typically require significant modifications to the model's architecture and often necessitate training a new model from scratch to achieve optimal performance. Given the substantial training costs associated with large language models, such retraining is not always feasible. While methods like StreamingLLM~\cite{xiao2023efficient} or Longformer~\cite{beltagy2020longformer} can be implemented through fine-tuning, their reliance on window attention compromises their ability to truly model long sequences, leading to a decrease in accuracy. This highlights the challenge of balancing model training efficiency with the ability to maintain high performance in handling long sequences.

\subsection{Frequency-based Transformers}

A various of research has focused on applying the Transformer architecture in the frequency domain. For instance, FNet~\cite{lee2021fnet} replaces the self-attention in BERT with Fourier Transform, significantly speeding up Transformer computations. A similar concept~\cite{buchholz2022fourier} has been adapted for image processing tasks. DCFormer~\cite{li2023discrete} proposes a Transformer-based network that learns semantic representations directly from frequency domain representations using Discrete Cosine Transform (DCT). In the realm of video prediction, ideas like the local frequency domain transformer~\cite{farazi2021local} have been introduced. However, applying these concepts to existing decoder-only large language models presents challenges. The auto-regressive inference style of these models makes token-level frequency domain transformations cumbersome. Each new token requires frequency domain transformation in conjunction with all previous tokens, which fails to reduce complexity and undermines the potential efficiency gains of frequency domain approaches in large-scale language models.

\begin{figure*}
	\centering
	\includegraphics[width=1.0\linewidth]{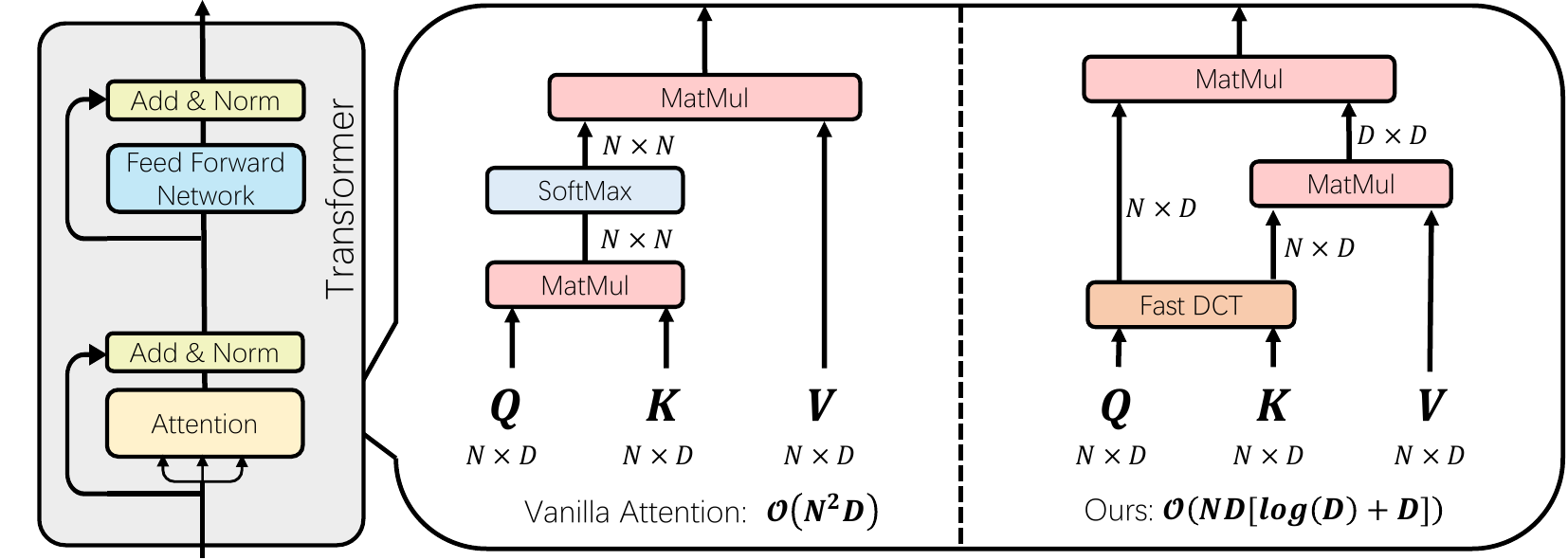}
	\vspace{-1.5em}
	\caption{Illustration of the proposed method, where the computation of queries and keys in the attention mechanism of a Transformer is efficiently mapped to the frequency domain using a fast Discrete Cosine Transform (DCT). This mapping effectively eliminates the softmax operation, thereby substantially reducing the computational complexity of the Transformer.}
	\label{fig:method}
	\vspace{-0.5em}
\end{figure*}

\section{Kernelized Attention in Frequency Domain}


In our study, we begin by revisiting the general form of self-attention~\cite{vaswani2017attention}. To simplify the notation and focus on the core aspects, we consider the single head form of self-attention and omit normalization factors. The self-attention mechanism is fundamentally composed of queries $Q$, keys $K$, and values $V$, expressed in the formula: 
\begin{equation}
	\begin{aligned}
		&\mbox{Attention}(Q, K, V) = \mbox{softmax}(QK^\intercal)V, \\
		&\mbox{where } Q,K,V\in\mathbb{R}^{n\times d},
	\end{aligned}
\label{attn}
\end{equation}
where $n$ denotes the number of tokens and $d$ denotes the hidden dimension of the attention. Specifically, when we denote $Q$ as $(q_1, q_2, ..., q_n)$, $K$ as $(k_1, k_2, ..., k_n)$, $V$ as $(v_1, v_2, ..., v_n)$, and output $O$ as $(o_1, o_2, ..., o_n)$, Equation~\ref{attn} can be reformulated as:
\begin{equation}
	\begin{aligned}
		&o_i = \sum_{j=1}^{n}\frac{e^{q_i k_j^\intercal}}{\sum_{j'=1}^n e^{q_i k_{j'}^\intercal}}v_j, \\
		&\mbox{where } q_i,k_i,v_i\in\mathbb{R}^{1\times d}, i=\{1,2,...,n\}.
	\end{aligned}
	\label{attn_1}
\end{equation}

It can be observed that the computational and memory complexity for calculating each output in a Transformer model is $\mathcal{O}(nd)$, where $n$ is the sequence length and $d$ is the dimensionality of the representation. Consequently, the time and memory complexity for processing a sentence of length $n$ scales quadratically, becoming $\mathcal{O}(n^2d)$. This quadratic scaling poses a significant computational burden, particularly for longer sequences where $n$ is large, making processing resource-intensive and challenging.

To mitigate this complexity, the concept of a kernel mechanism has been introduced as a means to reduce the computational demands of attention mechanisms, which has been introduced in~\cite{tsai2019transformer,katharopoulos2020transformers,choromanski2020rethinking}. Specifically, this involves the introduction of a kernel function ${K}(\cdot,\cdot)$, which acts as a positive-definite kernel capable of measuring similarity. By utilizing this kernel, the attention mechanism can be reformulated as:
\begin{equation}
		o_i = \sum_{j=1}^{n}\frac{{K}(q_i,k_j)}{\sum_{j'=1}^n \mathcal{K}(q_i,k_{j'})}v_j, 
	\label{kerattn}
\end{equation}

By applying the kernel trick, it's possible to linearly decompose the attention mechanism:
\begin{equation}
	o_i = \sum_{j=1}^{n}\frac{\phi(q_i)\phi(k_j)^\intercal}{\sum_{j'=1}^n \phi(q_i)\phi(k_{j'})^\intercal}v_j, 
	\label{linattn}
\end{equation}
where $\phi(\cdot):\mathbb{R}^d\rightarrow \mathbb{R}^m$ is a projection to map the inputs into $m$ dimension features. This decomposition benefits from the fact that the computational dimensions of the keys and values can be merged, effectively reducing the computational complexity from $\mathcal{O}(n^2d)$ to $\mathcal{O}(nmd)$. Given that the dimensionality $d$ and $m$ is typically much smaller than the sequence length $n$, this linearization of the attention mechanism results in a substantial decrease in computational intensity.

In the context of large language models, the cost of retraining is prohibitively high. In such scenarios, it becomes imperative to find a kernel that can equivalently replace the vanilla attention mechanism without necessitating extensive retraining. Positive Random Features (PRF)~\cite{choromanski2020rethinking} emerge as a viable candidate in this regard:
\begin{equation}
	\phi_{\mbox{\tiny PRF} }(x) = e^{\omega x^\intercal-\frac{\Vert x\Vert^2}{2}},
	\label{PRF}
\end{equation}
where $\omega\in\mathbb{R}^{m\times d}$. Theoretical demonstrations have established that $e^{q k^\intercal} = \mathbb{E}_{\omega \sim \mathcal{N}(0.I)} [e^{\omega q^\intercal-\frac{\Vert q\Vert^2}{2}} e^{\omega k^\intercal-\frac{\Vert k\Vert^2}{2}} ]$. It means that when $m$, the dimension of the feature space, is sufficiently large, Positive Random Features (PRF) mapping becomes an equivalent of the original attention mechanism. This equivalence suggests that, in theory, it is feasible to directly transform existing vanilla attention into linear attention using PRF mapping, thereby achieving an acceleration without loss of functionality. However, a notable challenge arises due to the need for $m$ to be set to a significantly large value to maintain the performance by reducing the approximation error. This requirement leads to a non-negligible increase in computational demand. For instance, in the case of the Performer~\cite{choromanski2020rethinking}, to achieve a lossless linear attention, $m$ often needs to be set to larger than $d$, diminishing the benefits of reduced computational load brought by linear attention.

To address this issue, we first conduct a theoretical analysis of the kernel-based approach for approximating attention mechanisms. We begin with the application of Bochner’s Theorem. This theorem allows us to equate the original attention computation involving queries (Q) and keys (K) – specifically the Gaussian kernel – to an integral computation akin to Equation~\ref{linattn}. 
\begin{theorem} 
	(Bochner's Theorem)~\cite{feller1966introduction}. A continuous shift invariant scaled kernel function $K(x,z): \mathbb{R}^d \to R$ is positive definite if and only if it is the Fourier Transform of a unique finite probability measure $p$ on $\mathbb{R}^d$.
	\begin{equation}
		\small
		K(x,z) = \int_{\mathbb{R}^d} e^{i(x-z)^\intercal w}p(w)dw = E_{w \sim p(\cdot)}[e^{iw^\intercal x}(e^{iw^\intercal z})^*],
		\label{kernel}
	\end{equation}
	where the symbol $z^*$ denotes the complex conjugate of $z$.
\end{theorem} 


According to Bochner's theorem, there is a one-to-one correspondence between the kernel function $K(x, z)$ and the probability density $p(w)$ defined on $\mathbb{R}^d$. Monte Carlo is equal weight approximation to kernel integrals. Taking $\varphi_p(x) :=\frac{1}{\sqrt{m}}[e^{-iw^\intercal _1x},...,e^{-iw^\intercal _mx}] ^\intercal $,  the feature maps can be constructed as:
\begin{equation}
	K(x,z) = E_{w \sim p(\cdot)}[e^{iw^\intercal x}(e^{iw^\intercal z})^*]\approx \varphi_p(x)^\intercal \varphi_p^*(z),
	\label{kernel-approx}
\end{equation}

where $w_i\sim p(\cdot) $ are samples constructed by Monte Carlo methods. $\varphi_p(\cdot)$ is the explicit finite dimensional feature map, which depends on the kernel $K$. Moving forward, instead of employing the Monte Carlo method as suggested in ~\cite{choromanski2020rethinking}, we utilize the Quasi-Monte Carlo method~\cite{le2013fastfood}. This shift enables the estimation of the integral using a specific uniform distribution as opposed to a randomly sampled distribution. 

Utilizing Bochner's theorem allows for a transformative interpretation of the attention mechanism in Transformer models. For the Gaussian Kernel:
\begin{equation}
	K_G(x, y) := e^{-\frac{\Vert x-y\Vert^2}{2}} = e^{-\frac{\Vert x\Vert^2+\Vert y\Vert^2}{2}}e^{x^\intercal y},\\
	\label{Gauss}
\end{equation}
since the $x$ and $y$ in attention mechanism is usually normalized, the Gaussian Kernel can be regarded as $e^{x^\intercal y}$, which is the same as the calculation between the queries and keys.

\begin{theorem}
	The Positive Fixed Features (PFF) is formulated as:
	\begin{equation}
		\varphi_{\mbox{\tiny PFF}}(x) :=\frac{e^{-\Vert x\Vert^2}}{\sqrt{m}}[e^{\Phi^{-1}(t_1)x^\intercal v_1},...,e^{\Phi^{-1}(t_m)x^\intercal v_m}] ^\intercal,
		\label{PFF}
	\end{equation}
	where $V=[v_1,...,v_m]\in \mathbb{S}^{d\times m}$ is asymptotically uniformly distributed and $t_i\sim U(0,1)$. Then, $\varphi_{\mbox{\tiny PFF}}(x)^\intercal \varphi_{\mbox{\tiny PFF}}(z)$ is an unbiased estimate of Gaussian kernel $K_G(x, y)$.
\end{theorem} 

The proof of this theorem involves a transformation to spherical coordinates, which can be found in the supplementary material. Through this transformation, we demonstrate that an approximation based on any asymptotically uniformly distribution can closely approximate the original Gaussian kernel. Furthermore, according to~\cite{asmussen2007stochastic}, when utilizing uniform sequences, the Quasi-Monte Carlo method can offer superior approximation efficiency compared to the traditional Monte Carlo method. The approximation efficiency of Quasi-Monte Carlo is $\mathcal{O}(1/m)$, which is more favorable than the $\mathcal{O}(1/m^{-0.5})$ efficiency of Monte Carlo. Consequently, this implies that using the PFF~\ref{PFF} kernel for approximating the Gaussian kernel is more advantageous than the PRF kernel in Equation~\ref{PRF}.

\begin{theorem}
	The Weighted Positive Fixed Features (WPFF) is formulated as:
	\begin{equation}
		\small
		\varphi_{\mbox{\tiny WPFF}}(x) :=\frac{De^{-\Vert x\Vert^2}}{\sqrt{m}}[e^{\Phi^{-1}(t_1)x^\intercal v_1},...,e^{\Phi^{-1}(t_m)x^\intercal v_m}] ^\intercal,
		\label{WPFF}
	\end{equation}
	where $D$ is a learnable parameter which can be optimized by the input $x$. Then the upper bound of the integral estimation error of the objective function by WPFF (Weighted Positive Fixed Features) method is not greater than the upper bound of the integral estimation error of the objective function by PFF (Positive Fixed Features) method. 
\end{theorem}   

Building upon the Quasi-Monte Carlo foundation, we further introduce the concept of weighted Quasi-Monte Carlo to enhance the efficiency of approximation. This advancement aims to leverage the strengths of the Quasi-Monte Carlo method, augmenting it with strategically weighted sampling to improve the precision and convergence rates of our approximations. The detailed proof is provided in the supplementary materials.

To further accelerate the training speed, we propose the use of frequency domain transformations to reduce the required computational resources. Fast Fourier Transform (FFT) and Discrete Cosine Transform (DCT) are commonly used methods for such transformations. Compared to ordinary orthogonal transformations, frequency domain transformations have algorithms for rapid computation, significantly reducing the computational cost of our proposed mapping. Specifically, the complexity of $\mathcal{O}(m)$ can be reduced to $\mathcal{O}(\log(m))$. Additionally, since DCT operates in the real number domain, it demands even less computational resources and is more hardware-friendly. Therefore, we opt for the DCT to carry out our kernel mapping.

\begin{algorithm}[tb]
	\caption{Frequency domain kernelization for efficient language models.}
	\label{alg}
	\begin{algorithmic}
		\INPUT A small amount of data $x_i$, a pre-trained Transformer model $M$.
		\STATE {\bfseries 1. Initialization:} the DCT coefficient $\mathcal{C}$, the weight $D$, the diagonal matrix $T$ in Equation~\ref{DCF} for each layer in $M$.
		\STATE {\bfseries 2. Transformation:} transform the vanilla attention calculation $\mbox{Attention}(Q, K, V) = \mbox{softmax}(QK^\intercal)V$ to $\mbox{FKA}(Q, K, V) = \phi_{\mbox{\tiny WDCF}}(Q)\phi_{\mbox{\tiny WDCF}}(K)^\intercal V$ using the Weighted Discrete Cosine Features for each layer in $M$.
		\STATE {\bfseries 3.} Get the transformed model $M_{\mbox{\tiny FKA}}$.
		\REPEAT
		 \STATE {\bfseries 4.} Randomly select a batch of data from $x_i$.
		\STATE {\bfseries 5.} Employ the transformed model $M_{\mbox{\tiny FKA}}$ on the mini-batch.
		\STATE {\bfseries 6.} Update weights in $M_{\mbox{\tiny FKA}}$ according to the loss and gradient;
		\UNTIL{convergence.}
	    \OUTPUT An efficient language model $M_{\mbox{\tiny FKA}}$.
	\end{algorithmic}
\end{algorithm}

Specifically, a DCT coefficient $\mathcal{C}\in\mathbb{R}^{d\times d}$ in the frequency domain is defined as:
\begin{equation}
	\scriptsize
	\mathcal{C}_{j_1 j_2} = s_{j_1}s_{j_2} \sum_{i_1=0}^{n-1} \sum_{i_2=0}^{d-1} \cos \left(\frac{\pi (2i_1+1)j_1}{2d}\right) \cos \left(\frac{\pi (2i_2+1)j_2}{2d}\right),
	\label{DCT}
\end{equation}
where $s_j = \sqrt{1/d}$ if $j=0$ and $s_j = \sqrt{2/d}$ otherwise. The weighted mapping using DCT (which is called Weighted Discrete Cosine Features) can be reformulated as:
\begin{equation}
	\phi_{\mbox{\tiny WDCF}}(x) = D e^{T \mathcal{C} x^\intercal},
	\label{DCF}
\end{equation}
where $\mathcal{C}\in\mathbb{R}^{m\times d}$ is the DCT coefficient, $D\in\mathbb{R}^{m}$ is a learnable weight, and $T= \mbox{diag}(t_1,\dots,t_m)$ is a random diagonal matrix following the inverse cumulative distribution. Note that since the $x$ in attention mechanism is usually normalized, we ignore the term of $\Vert x\Vert^2$ in Equation~\ref{PFF} for efficiency. Therefore, using DCT as a kernel can closely approximate the original attention mechanism while have low computation complexity. For scenarios where $m>d$, more DCT transformations can be derived using different boundary conditions. Details can be referred to~\cite{ahmed1974discrete}. It is noted that we set $m=d$ to avoid increasing computational complexity in the subsequent experiments.

\begin{table*}
	\centering
	\caption{The experimental results of the proposed method. Training time is measured using A800. Inference throughput is evaluated with token length of 2048. * denotes results from~\cite{he2024densemamba}. }
	\begin{tabular}{ccccccc|ccc}
		\toprule
		\multirow{2}{*}{Model}  & \multirow{2}{*}{PIQA }& \multirow{2}{*}{WinoGrande}   & \multirow{2}{*}{WSC} & \multirow{2}{*}{ARC-E} & \multirow{2}{*}{ARC-C} & \multirow{2}{*}{LogiQA} & \multirow{2}{*}{Avg} & Training & Inference  \\
		&  &    &  &  &  &  &  & (day) & (tokens/s)\\
		\midrule
		Pythia-70M   &
		0.498&
		0.484&
		0.596&
		0.25&
		0.221&
		0.202&
		0.375&
		21.3 & 2037  \\
		DiJiang-70M  & 0.587&
		0.511&
		0.365&
		0.403&
		0.213&
		0.253&
		0.389&
		1.3 &2605 \\\midrule
		Pythia-160M&
		0.532&
		0.484&
		0.634&
		0.265&
		0.227&
		0.202&
		0.391&42.9 &622\\
		DiJiang-160M&
		0.618&
		0.490&
		0.384&
		0.439&
		0.217&
		0.239&
		0.398&2.7&1315\\\midrule
		Pythia-410M&
		0.668&
		0.537&
		0.567&
		0.521&
		0.213&
		0.22&
		0.454&105.8& 203\\
		DiJiang-410M&
		0.663&
		0.524&
		0.567&
		0.492&
		0.244&
		0.247&
		0.456&6.6&787\\\midrule
		Pythia-1B&
		0.706&
		0.533&
		0.365&
		0.569&
		0.269&
		0.296&
		0.456&201.2&105\\
		Mamba-1.3B*&
	    0.663&
		0.530&
		0.365&
		0.508&
		0.251&
		0.263&
		0.430
		&-&-\\
		DiJiang-1B&
		0.677&
		0.521&
		0.365&
		0.537&
		0.253&
		0.284&
		0.440&12.6&611\\\midrule
		Pythia-2.8B&
		0.737&
		0.596&
		0.384&
		0.640&
		0.295&
		0.215&
		0.478&593.3&34\\
		DiJiang-2.8B&
		0.713&
		0.545&
		0.413&
		0.597&
		0.289&
		0.279&
		0.473&37.1&284\\\midrule
		OPT-350M&
		0.645&
		0.524&
		0.365&
		0.441&
		0.208&
		0.210&
		0.399
		&-&201\\
		DiJiang-350M&
		0.550&
		0.507&
		0.635&
		0.286&
		0.227&
		0.223&
		0.404&5.6&820\\\midrule
		TinyLLaMA-1.1B&
		0.666&
		0.541&
		0.413&
		0.487&
		0.211&
		0.228&
		0.424
		&-&74\\
		DiJiang-1.1B&
		0.535&
		0.508&
		0.635&
		0.286&
		0.243&
		0.212&
		0.403&13.9&613\\
		\bottomrule
	\end{tabular}
	\label{tab:whole}
	\vspace{-0.5em}
\end{table*}

Therefore, the kernelized attention in frequency domain (FKA) is then reformulated as:
\begin{equation}
	\begin{aligned}
		&\mbox{FKA}(Q, K, V) = \phi_{\mbox{\tiny WDCF}}(Q)\phi_{\mbox{\tiny WDCF}}(K)^\intercal V, \\
		&\mbox{where } Q,K,V\in\mathbb{R}^{n\times d},
	\end{aligned}
	\label{fka_attn}
\end{equation}
This approach achieves a notable reduction in computational complexity by employing the Discrete Cosine Transform (DCT) to map the queries and keys within the Transformer's attention mechanism to a domain where operations are inherently more efficient.

In summary, our method leverages frequency domain kernelization for Transformer attention mechanisms, significantly cutting computational costs while either preserving or enhancing model performance. The details are shown in  Algorithm~\ref{alg}. Through the strategic use of the weighted Quasi-Monte Carlo method, which outperforms traditional Monte Carlo sampling in efficiency and accuracy, combined with DCT for efficient frequency domain transformations, we attain linear complexity in attention computation. This reformulation not only improves the scalability of Transformers, enabling them to handle larger datasets and extended sequences with ease, but also markedly accelerates the training and inference phases.

\section{Experiments}

In this section, we conduct extensive experimental validation of the proposed architecture, encompassing results across language models of varying scales. Additionally, we provide detailed analyses to substantiate the effectiveness of our approach.

\subsection{Evaluation on Different Scales}
Given the challenge of replicating the training processes of most language models, as only their checkpoints are openly available, we opted to validate our method using Pythia~\cite{biderman2023pythia}, a model with a fully public dataset and training procedure, enabling fair comparisons. 

We adhered to the exact training settings employed by Pythia, including learning rates, optimizers, and other hyperparameters, and utilized the Pile dataset. The Pile~\cite{gao2020pile} is an 825 GiB corpus of English text, specifically designed for training large-scale language models. This project is composed of 22 distinct, high-quality subsets, both pre-existing and newly constructed, many of which originate from academic or professional sources. This comprehensive and diverse dataset serves as a robust foundation for developing and fine-tuning language models Our DiJiang model was fine-tuned from the pre-trained Pythia model. We evaluated our approach on six public datasets used by Pythia: PIQA~\cite{bisk2020piqa}, WinoGrande, WSC~\cite{sakaguchi2021winogrande}, ARC-E, ARC-C~\cite{clark2018think}, and LogiQA~\cite{liu2020logiqa}. The Pythia model's checkpoint was obtained from HuggingFace\footnote{\url{https://huggingface.co/EleutherAI}}. We adapt the learned gating mechanism~\cite{peng2021random} similar with the RetNet~\cite{sun2023retentive} to augment our DiJiang.

The experimental results, as shown in Table~\ref{tab:whole}, indicate that our method achieved remarkable outcomes across different model sizes, ranging from 70M to 2.8B parameters. On average, the performance on the six datasets was nearly identical to that of the original Pythia, but with only $\sim1/16$ of the training cost. Furthermore, the inference speed of our DiJiang model was significantly faster than that of the original Pythia. These results substantiate the effectiveness of our approach, demonstrating its potential to enhance the efficiency of large language models without compromising performance.

\subsection{Evaluation on Different Models}

To evaluate the effectiveness of our method across different models, as shown in Table~\ref{tab:whole}, we further applied our approach to the OPT-350M~\cite{zhang2022opt}\footnote{\url{https://huggingface.co/facebook/opt-350m}} and TinyLLaMA-1.1B\footnote{\url{https://huggingface.co/TinyLlama/TinyLlama-1.1B-python-v0.1}} models. It's important to note that since their training data are not fully accessible, we continued to use the Pile dataset for fine-tuning them. 

Finally, we conducted further experiments on the well-known publicly available large language model, LLaMA2-7B, fine-tuning it into the DiJiang-7B model. Table~\ref{tab:llama27b} reveal that the DiJiang-7B model achieves results that are virtually identical to the original LLaMA2-7B across various benchmarks. Remarkably, our model required only 40B training data, significantly less than the 2T tokens used by LLaMA2-7B. This demonstrates the successful application of our method to large-scale models at the 7B parameter level, highlighting the efficiency and effectiveness of our fine-tuning approach even when scaling to vast model sizes.

Interestingly, we found that despite using a limited dataset, our method achieved results similar to the original models with a significantly lower training cost and faster speed. This outcome further demonstrates the strong generalizability and flexibility of our approach, underscoring its potential applicability across a broad spectrum of language models, even in scenarios where the original training datasets are not available.

\begin{table*}
	\centering
	\caption{Comparison of different linear attention models on fine-tuning Pythoia-410M~\cite{biderman2023pythia}. }
	\begin{tabular}{ccccccc|c}
		\toprule
		Model  & PIQA & WinoGrande   & WSC & ARC-E & ARC-C & LogiQA & Avg   \\
		\midrule
		Pythia-410M~\cite{biderman2023pythia} &
		0.668&
		0.537&
		0.567&
		0.521&
		0.213&
		0.22&
		0.454\\
		\midrule
		Linformer~\cite{wang2020linformer}&
		0.5267&
		0.5114&
		0.6346&
		0.2656&
		0.244&
		0.2074&
		0.3982
		\\
		Cosformer~\cite{qin2022cosformer}&
		0.5218&
		0.5059&
		0.6058&
		0.2673&
		0.2637&
		0.2642&
		0.4047
		\\
		Performer~\cite{choromanski2020rethinking}&
		0.6431&
		0.4964&
		0.4327&
		0.4701&
		0.2312&
		0.2366&
		0.4183
		\\
		RetNet~\cite{sun2023retentive}&
		0.4951&
		0.4957&
		0.6346&
		0.2508&
		0.227&
		0.2028&
		0.3843
		\\\midrule
		PFF (Equation~\ref{PFF})&
		0.6453&
		0.4996&
		0.4712&
		0.4747&
		0.2295&
		0.2381&
		0.4264
		\\
		
		DiJiang (Ours)&
		0.6638&
		0.5241&
		0.5673&
		0.4928&
		0.2449&
		0.2473&
		\textbf{0.4567}
		\\
		\bottomrule
	\end{tabular}
	\label{tab:comp}
	\vspace{-0.5em}
\end{table*}

\begin{table*}
	\scriptsize
	\centering
	\caption{Comparison with LLaMA2-7B on various benchmarks. }
	\begin{tabular}{cccccccccccc|cc}
		\toprule
		Model  & PIQA & SIQA & BoolQ & WSC & HellaSwag & ARC-E & ARC-C & MMLU & NQ & COPA & Race-Middle & Avg & Tokens  \\
		\midrule
		LLaMA2-7B &
		0.782&
		0.485&
		0.749&
		0.663&
		0.740&
		0.561&
		0.403&
		0.468&
		0.192&
		0.670&
		0.402&
		0.565&
		2000B\\
		DiJiang-7B&
		0.775&
		0.346&
		0.626&
		0.683&
		0.694&
		0.626&
		0.427&
		0.407&
		0.194&
		0.730&
		0.618&
		0.557&
		40B\\
		\bottomrule
	\end{tabular}
	\label{tab:llama27b}
	\vspace{-0.5em}
\end{table*}

\subsection{Comparison with Linear Transformers}

To compare the superiority of our approach against other linear-complexity self-attention Transformer models, we validated the fine-tuning results on Pythia-400M for different models including Linformer, Performer, RetNet, and Cosformer. For a fair comparison, we employed the same training settings and data. Table~\ref{tab:comp} displays the comparative results, revealing that while existing methods can achieve good results through retraining, as evidenced by their original publications, most of them suffer from significant accuracy losses in scenarios where fine-tuning is done without retraining. This is largely because these methods struggle to accurately approximate the original attention mechanism, leading to an inability to restore the original accuracy with minimal training.

Among these comparison methods, Performer achieved the best results by approximating the original attention with Positive Random Features (PRF). However, as previously discussed, this Monte Carlo-based approximation method cannot achieve satisfactory outcomes, resulting in accuracy loss. By switching from Monte Carlo to the Quasi-Monte Carlo scheme using Positive Fixed Features (PFF) as described in Equation~\ref{PFF}, we surpassed the accuracy of Performer but still fell short of the original vanilla Transformer's performance. Furthermore, by incorporating the Discrete Cosine Transform (DCT), our method achieves higher efficiency than approaches using PFF kernels. The DCT transformation enables a more compact and efficient representation of the frequency components of the attention mechanism. This efficiency stems from the DCT's ability to concentrate energy, allowing for a sparse representation that captures the most significant features of the data with fewer coefficients. Consequently, our approach not only closely approximates the original attention but also does so with improved computational performance compared to PFF-based methods. This advantage highlights the effectiveness of using DCT in optimizing the approximation of attention mechanisms, further underscoring the potential of our method in enhancing the efficiency of Transformer models. Further incorporating weighted Quasi-Monte Carlo, our DiJiang architecture ultimately achieved accuracy nearly identical to the original Pythia-400M, validating the efficacy of our approximation method. This demonstrates not only the potential of our approach for fine-tuning large-scale language models but also underscores the importance of choosing an efficient approximation strategy to maintain model performance.

\begin{figure}[ht]
	\vspace{-0.5em}
	\centering
	\includegraphics[width=1.0\linewidth]{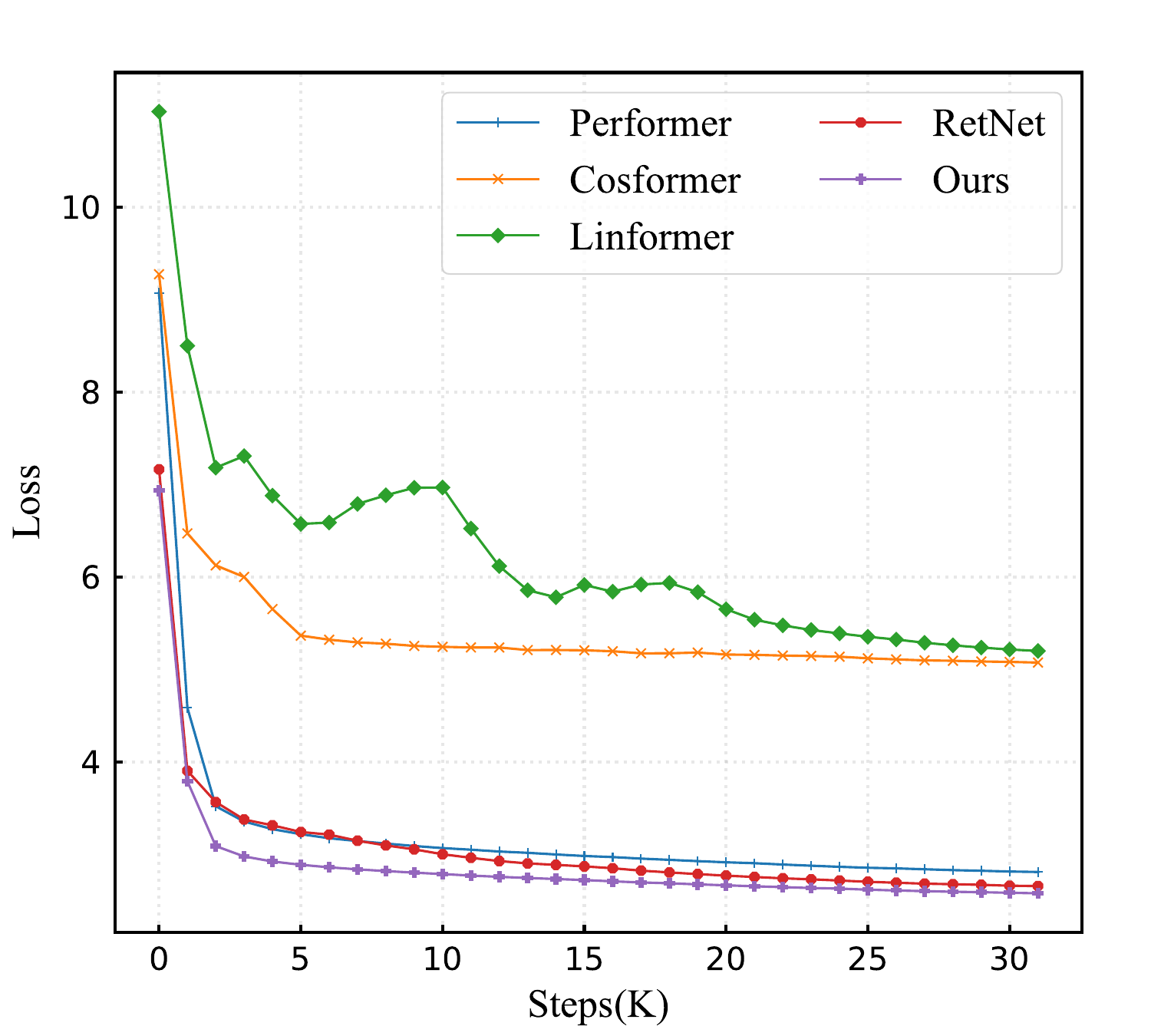}
	\vspace{-2.0em}
	\caption{Training Curve of different methods. The proposed method achieves the lowest PPL and the fastest converge speed.}
	\vspace{-0.5em}
	\label{fig:curve}
\end{figure}

\begin{figure*}[t]
	\centering
	\vspace{-0.5em}
	\begin{tabular}{ccc}
		\includegraphics[width=0.31\linewidth,trim=2cm 0cm 2cm 0cm]{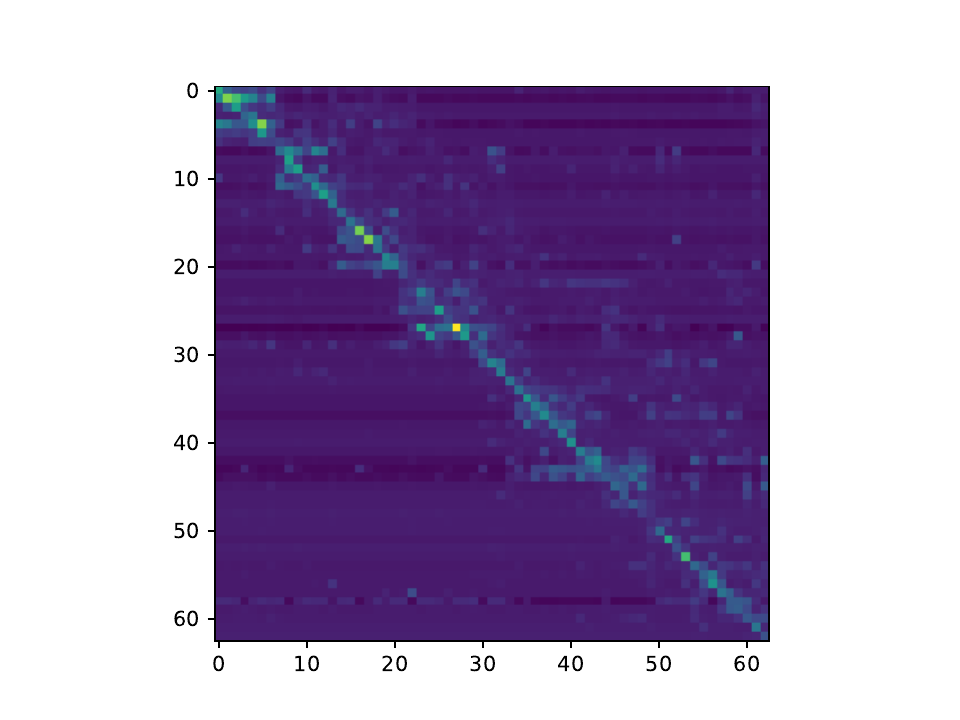}&\includegraphics[width=0.31\linewidth,trim=2cm 0cm 2cm 0cm]{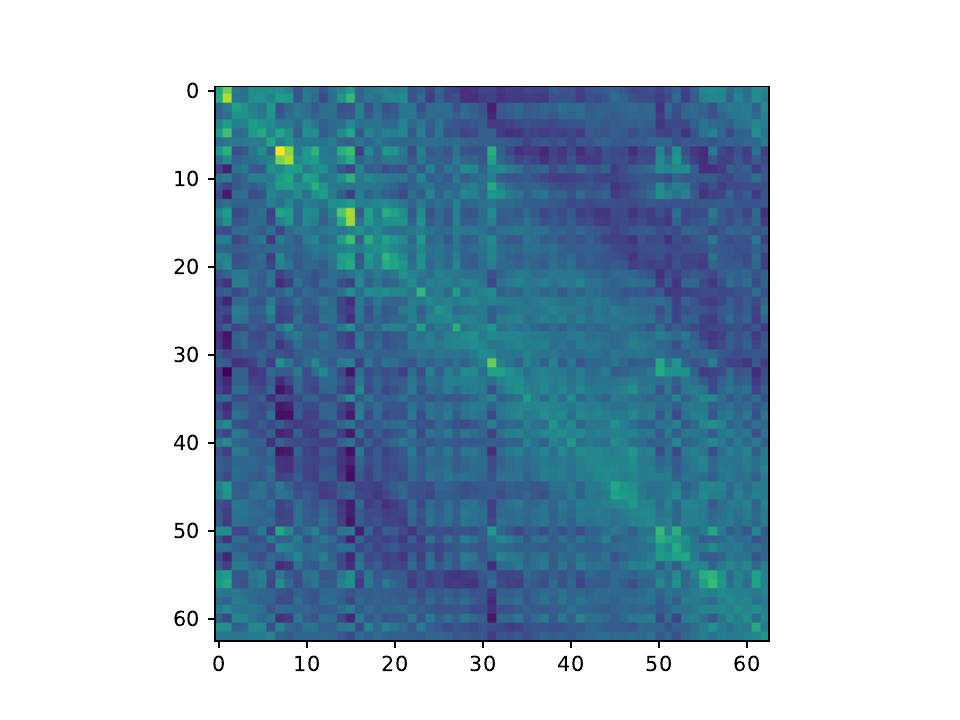}&\includegraphics[width=0.31\linewidth,trim=2cm 0cm 2cm 0cm]{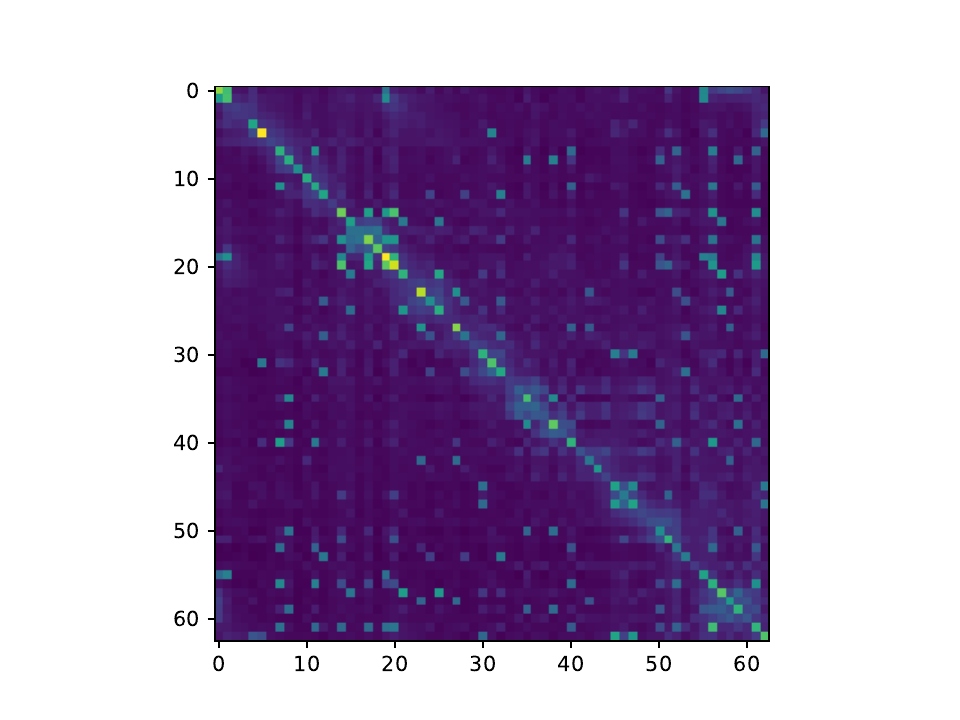}\\
		\vspace{-0.5em}
		(a) Vanilla attention & (b) PRF kernel  &(c) WDCF kernel\\
	\end{tabular}
\vspace{-0.5em}
	\caption{Visualization of attention map of different architectures. The results are averaged by multiple heads.}
	\vspace{-0.5em}
	\label{fig:attn}
\end{figure*}

\begin{figure}[!h]
	\centering
	\vspace{-0.5em}
	\includegraphics[width=0.9\linewidth]{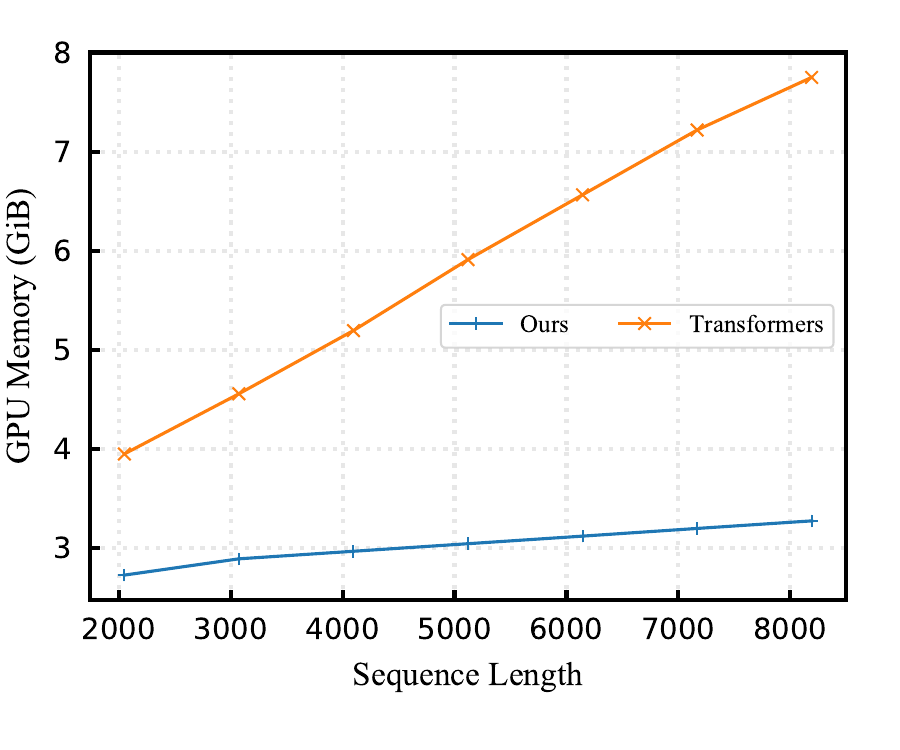}\\
	\vspace{-1.0em}
	(a) Inference Memory\\
	\includegraphics[width=0.9\linewidth]{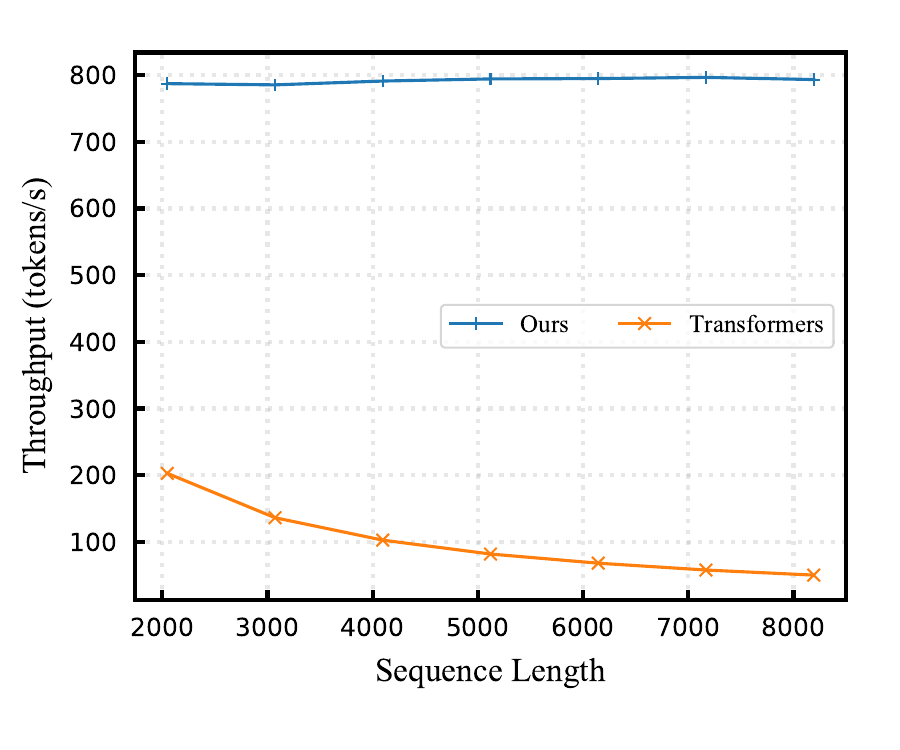}\\
	\vspace{-1.0em}
	(b) Inference Throughput
	\vspace{-0.5em}
	\caption{Comparison of inference memory and throughput between the proposed DiJIang and vanilla Transformer architecture.}
	\vspace{-0.5em}
	\label{fig:speed}
\end{figure}

We further visualized the training curves to showcase the approximation efficiency of different linear Transformer models, as depicted in Figure~\ref{fig:curve}. RetNet, as an emerging language model architecture, has shown its potential by achieving significantly low loss values, underscoring its capability for language tasks. Despite its low loss, RetNet does not necessarily outperform on benchmark metrics and, in some cases, even falls short of the results achieved by the Performer. This discrepancy highlights the importance and advantages of employing kernel methods to approximate the original attention computation, particularly in fine-tuning scenarios. 

Our method demonstrates the fastest rate of loss reduction and ultimately achieves the lowest loss value. This rapid convergence indicates that our approach can quickly reach a performance level similar to that of the original Transformer. The visualization clearly underscores the superiority of our method in terms of both convergence speed and final model accuracy, validating our approach's effectiveness in efficiently approximating the attention mechanism while maintaining high performance standards. This visual evidence further solidifies our claim that our method stands out among linear Transformer alternatives, offering a compelling solution for optimizing Transformer models without compromising on quality.

\subsection{Comparison of Inference Cost}

Furthermore, we also evaluated the memory usage and throughput of our method in comparison to the original Transformer model under various conditions. We selected the Pythia-410M model as our primary subject for analysis. We follow the implementation of RetNet~\cite{sun2023retentive} to efficient inference. The specific results, as depicted in Figure~\ref{fig:speed}, demonstrate that as the token length increases, the memory footprint and inference speed of our model do not escalate. This observation is attributed to the linear complexity characteristic of our approach, indicating that our method is more conducive to long-sequence inference. In contrast, due to the quadratic complexity of attention computations, the original Transformer model experiences a continuous increase in both inference time and required memory as the token length grows. This comparison highlights the efficiency and practicality of our solution, particularly in scenarios involving extensive sequences where computational resources are a critical concern.

\subsection{Visualization}
To further demonstrate the effectiveness of our model's approximation of the attention mechanism, we present attention maps generated by different methods in Figure~\ref{fig:attn}. It is evident that the original Transformer's attention map (Figure~\ref{fig:attn} (a)) is rich in information, laying the foundation for its robust capabilities. In contrast, attention maps produced by other linear attention methods such as Performer (Figure~\ref{fig:attn} (b)) fail to adequately capture the relationships between tokens, resulting in maps that are dissimilar to those of the original Transformer and ultimately leading to decreased model accuracy, despite fine-tuning efforts. In contrast, our method  (Figure~\ref{fig:attn} (c)), by employing the weighted Quasi-Monte Carlo scheme, closely approximates the original attention mechanism. This allows it to effectively model the relationships between different tokens, achieving results nearly identical to those of the original Transformer but with significantly faster inference efficiency. This comparison not only highlights the inadequacies of other linear attention methods in capturing token interdependencies but also showcases the superiority of our approach in accurately approximating attention while enhancing computational efficiency.

\section{Conclusion}
This paper introduces DiJiang, a groundbreaking Frequency Domain Kernelization method designed to address the computational inefficiencies inherent in traditional Transformer models. By leveraging linear attention mechanisms and a novel application of the weighted Quasi-Monte Carlo method for efficient sampling, our approach significantly reduces the necessity for extensive retraining. This is particularly beneficial for large language models, where the cost and time associated with training are substantial barriers to progress. The kernelization process, underpinned by Discrete Cosine Transform (DCT), not only diminishes the computational complexity but also ensures that the adaptation from a vanilla Transformer to a linear attention model incurs minimal training costs. Our extensive experiments validate that DiJiang achieves performance on par with conventional Transformers while reducing training costs by about 10x and enhancing inference speeds. This method represents a significant advancement in the development of efficient and scalable Transformer models, promising wider applicability and facilitating advancements in various tasks within the realm of natural language processing and beyond.

\section*{Broader Impact}

This paper presents work whose goal is to advance the field of Machine Learning. There are many potential societal consequences of our work, none which we feel must be specifically highlighted here.


\bibliography{example_paper}
\bibliographystyle{icml2024}

\newpage
\appendix
\onecolumn
\section{Theoretical Proof.}
\begin{theorem}
	The Positive Fixed Features (PFF) is formulated as:
	\begin{equation}
		\varphi_{\mbox{\tiny PFF}}(x) :=\frac{e^{-\Vert x\Vert^2}}{\sqrt{m}}[e^{\Phi^{-1}(t_1)x^\intercal v_1},...,e^{\Phi^{-1}(t_m)x^\intercal v_m}] ^\intercal,
	\end{equation}
	where $V=[v_1,...,v_m]\in \mathbb{S}^{d\times m}$ is asymptotically uniformly distributed and $t_i\sim U(0,1)$. Then, $\varphi_{\mbox{\tiny PFF}}(x)^\intercal \varphi_{\mbox{\tiny PFF}}(z)$ is an unbiased estimate of Gaussian kernel $K_G(x, y)$.
\end{theorem} 

\begin{proof}
	
The proof is motivated by~\cite{lyu2017spherical}. We also use spherical coordinate changes to get the following proof. The Gaussian kernel is real-valued and therefore the imaginary part in Eq.(\ref{kernel}) can be discarded.
\begin{equation}
		K_G(x,y) = e^{-\frac{\Vert x-y\Vert^2}{2}}
		= \int_{\mathbb{R}^d} e^{i(x-y)^\intercal w}\mu(w)dw
		= \int_{\mathbb{R}^d} \cos((x-y)^\intercal w)\mu(w)dw
		= e^{-\Vert x\Vert^2-\Vert y\Vert^2}\int_{\mathbb{R}^d} e^{((x+y)^\intercal w)}\mu(w)dw,
	\label{normal-distribution}
\end{equation}
where $\mu(\cdot)$ is the probability density function of d-dimensional standard normal distribution.

The Gaussian kernel is a shift and rotation invariant kernel. Given any rotation $R \in SO(d)$, where $SO(d)$ denotes rotation groups, the corresponding probability density is also Gaussian according to Bochner's theorem. For shift and rotation invariant kernels, we can convert the integral to spherical coordinates. $r = ||w||_2$ and $p(r)$ be the density function of $r$, and $w=rv$. Because of the rotation invariant property of $K_G(x,y)$, we achieve:
\begin{equation}
	\label{shift-and-rotation-invariant-kernel}
		K_G(x,y) = \int_{\mathbb{R}^d} e^{i(x-y)^\intercal w}p(w)dw 
		= \int_{R_+} \int_{\mathbb{S}^{d-1}}e^{i(x-y)^\intercal rv}p(r)drd\sigma (v)
		= \int_{[0,1]} \int_{\mathbb{S}^{d-1}}e^{i(x-y)^\intercal \Phi^{-1}(t)v} dtd\sigma (v),
\end{equation}
where $\sigma$ denotes the normalized surface area measure on $\mathbb{S}^d := \{x\in \mathbb{R}^{d}| \Vert x\Vert_2 = 1 \}$ and $\Phi^{-1}(t)$ denotes the inverse cumulative distribution function w.r.t is a non-negative radial scale.

For real valued continuous shift and rotation invariant scaled kernel $K_G(x,y)$, the imaginary parts of the
integral vanish. We can achieve:
\begin{equation}
	K_G(x,y) = \int_{\mathbb{R}^d} e^{i(x-z)^\intercal w}p(w)dw
		= \int_{\mathbb{R}^d} \cos((x-z)^\intercal w)p(w)dw
		= \int_{[0,1]} \int_{\mathbb{S}^{d-1}} \cos((x-z)^\intercal \Phi^{-1}(t)v) dtd\sigma (v).
	\label{shift_and_rotation_realkernel}
\end{equation}
For Gaussian kernel, we can get another medium integral form:
\begin{equation}
	\small
		K_G(x,y) = \int_{\mathbb{R}^d} e^{i(x-y)^\intercal w}\mu(w)dw
		= e^{-\Vert x\Vert^2-\Vert y\Vert^2}\int_{\mathbb{R}^d} e^{((x+y)^\intercal w)}\mu(w)dw
		= e^{-\Vert x\Vert^2-\Vert y\Vert^2}\int_{[0,1]} \int_{\mathbb{S}^{d-1}} e^{(x+y)^\intercal \Phi^{-1}(t)v}dtd\sigma (v).
	\label{shift_and_rotation_realkernel_Gauss}
\end{equation}
According to~\cite{brauchart2015distributing}, if the point set $V=[v_1,...,v_m]\in \mathbb{S}^{d\times m}$ is asymptotically uniformly distributed, the following equation holds true:
\begin{equation}
	\lim_{m \to \infty} \frac{1}{m} \sum_{i=1}^{m}f(v_i) = \int_{\mathbb{S}^d}f(v)d\sigma(v). 
	\label{asymptotically}
\end{equation}
Then, we have:
	\begin{equation}
		\begin{aligned}
			\lim_{m\to \infty}  E_{t_i\sim U(0,1)}[\varphi_{\mbox{\tiny PFF}}(x)^\intercal \varphi_{\mbox{\tiny PFF}}(y)] 
			=& E_{t_i\sim U(0,1)}[\lim_{m \to \infty} \frac{e^{-\Vert x\Vert^2-\Vert y\Vert^2}}{m} \sum_{i=1}^m  e^{\Phi^{-1}(t_j)(x+y)^\intercal v_i}] \\
			=&e^{-\Vert x\Vert^2-\Vert y\Vert^2}\int_{[0,1]} \int_{\mathbb{S}^{d-1}} e^{(x+y)^\intercal \Phi^{-1}(t)v}dtd\sigma (v)\\
			=&e^{-\Vert x\Vert^2-\Vert y\Vert^2}(2\pi)^{-\frac{d}{2}}\int_{\mathbb{R}^{d}} e^{(x+y)^\intercal w}e^{-\frac{\Vert w\Vert^2}{2}}dw\\
			=&e^{-\Vert x\Vert^2-\Vert y\Vert^2}(2\pi)^{-\frac{d}{2}}\int_{\mathbb{R}^{d}} e^{\frac{\Vert x+y\Vert^2}{2}}e^{-\frac{\Vert w-(x+y)\Vert^2}{2}}dw\\
			=&e^{-\Vert x\Vert^2-\Vert y\Vert^2}e^{\frac{\Vert x+y\Vert^2}{2}}\\
			=&e^{-\frac{\Vert x-y\Vert^2}{2}}\\
			=&K_G(x,y).\\
		\end{aligned}
		\label{SSF1}
	\end{equation}
Therefore, $\varphi_{\mbox{\tiny PFF}}(x)^\intercal \varphi_{\mbox{\tiny PFF}}(z)$ is an unbiased estimate of Gaussian kernel $K_G(x, y)$.
\end{proof}

\begin{theorem}
	The Weighted Positive Fixed Features (WPFF) is formulated as:
	\begin{equation}
		\varphi_{\mbox{\tiny WPFF}}(x) :=\frac{De^{-\Vert x\Vert^2}}{\sqrt{m}}[e^{\Phi^{-1}(t_1)x^\intercal v_1},...,e^{\Phi^{-1}(t_m)x^\intercal v_m}] ^\intercal,
	\end{equation}
	where $D$ is a learnable parameter which can be optimized by the input $x$. Then the upper bound of the integral estimation error of the objective function by WPFF (Weighted Positive Fixed Features) method is not greater than the upper bound of the integral estimation error of the objective function by PFF (Positive Fixed Features) method. 
\end{theorem}  

\begin{proof}
	
The proof is motivated by~\cite{yang2014quasi}. We use some of the same mathematical definitions and similar proofs from this paper to show that the WPFF method has a smaller upper bound on the overall estimation error of the objective function. Theorem A.2 Lemma A.3 and Lemma A.4 are all relevant to this paper.

Consider the task of computing an approximation of the following integral $I_{d,p}[f] = \int_{\mathbb{R}^d}f(x)p(x)dx$, because of $I_{d,p}[f]=E_{x \sim p(\mathbb{R}^d)}[f(x)]$, an empirical approximation called Monte Carlo (MC) to the integral can be computed by drawing a random point set $S=\{w_1,...,w_s\}$ independently from $p(\mathbb{R}^d)$. When $S$ is a set of fixed points, the empirical approximation is a quasi-Monte Carlo (QMC) method. The purpose of the QMC method is to improve convergence speed by constructing $S$ using deterministic low-differential sequences instead of random sampling points. We have $I_S[f] = \frac{1}{s}\sum_{w \in S}f(w)$.

We define the integration error with respect to the point set S as $\epsilon_S[f] = |I_{d,p}(f)-I_S(f)|$. The integration error for PFF is as follows:
\begin{equation}
		\epsilon_{S,p}[f] = |\int_{\mathbb{R}^d}f(x)p(x)dx - \frac{1}{s}\sum_{j=1}^s f(w_j)|,
\end{equation}
where $S$ is a set of fixed points.

The classical Monte Carlo and quasi-Monte Carlo approximations of integrals have consistent weights. However, it makes sense to weight the approximations, approximate $I_{d,p}[f] = \int_{\mathbb{R}^d}f(x)p(x)dx$ using $I_{S,\Xi}[f] = \sum_{j= 1}^{s}\xi_jf(w_j)$, where $\Xi = \{\xi_1, \xi_2, ..., \xi_s \}$, $\xi_i \geq 0$ for $i \in \{1,2,...,s\}$, we do not need to normalize the weights, that is it is possible that $\sum_{i=1}^s \xi_i \neq 1$.
\begin{equation}
	\begin{aligned}
		\int_{\mathbb{R}^d}e^{i(x-z)^Tw}p(w)dw \approx& \sum_{j = 1}^{s}\zeta_j(x)\zeta_j(z)f(w_j)\\
		=& \sum_{j = 1}^{s}\zeta_j(x)\zeta_j(z)e^{ix^Tw_j}e^{-iz^Tw_j}\\
		=& \varphi_{\mbox{\tiny WPFF},S}(x)^T\varphi_{\mbox{\tiny WPFF},S}(z),
	\end{aligned}
	\label{weighted_QMC_kernel}
\end{equation}
where $\zeta_j(x)\zeta_j(z)=\xi_j$, $\zeta_j(x),\zeta_j(z) \geq 0$, $e^{ix^Tw_j}e^{-iz^Tw_j} = f(w_j)$ for $j \in \{1,2,...,s\}$, and $\Psi_{S}(x) = [\zeta_1(x)e^{ix^Tw_1},...,\zeta_s(x)e^{ix^Tw_s}]^T$.

The integration error for WPFF is as follows:
\begin{equation}
	\begin{aligned}
		\epsilon_{S,p,\Xi}[f] = |\int_{\mathbb{R}^d}f(x)p(x)dx - \sum_{j=1}^s \xi_jf(w_j)|.
	\end{aligned}
\end{equation}
For a vector $b \in R_d$, let us define $\square b = \{u\in \mathbb{R}^d | |u_j| \leq |b_j|\}$. Let
\begin{equation}
	\begin{aligned}
		F_{\square b} =\{f_u(x) = e^{ix^Tu} | u\in \square b\},
	\end{aligned}
	\label{Fb}
\end{equation}
and consider the space of functions that admit an integral representation over $F_{\square b}$ of the form
\begin{equation}
	\begin{aligned}
		f(x) = \int_{u \in \square b} \hat{f}(u)e^{ix^Tu}du,
	\end{aligned}
	\label{Fb_sq}
\end{equation}
where $\hat{f}(u) \in \ell_2(\square b)$. This space is associated with the functions with compactly-supported inverse Fourier transforms called bandlimited functions, which play an important role in Shannon-Nyquist sampling theory. Under a natural choice of inner product, these spaces are called Paley-Wiener spaces and they constitute an RKHS. 
\begin{lemma} 
	(The Kernel of Paley-Wiener RKHS) According to \cite{peloso2011classical}, $PW_b$ denotes the space of functions which are represented in the form of Eq.\ref{Fb_sq}, with the inner product $<f,g>_{PW_b}={(2\pi)}^{2d}<\hat{f},\hat{g}>_{L_2(\square b)}$. $PW_b$ is an RKHS with kernel function,
	\begin{equation}
		\sin c_b(u,v)=\pi^{-d} \prod_{i=1}^d \frac{sin(b_j(u_j-v_j))}{u_j-v_j}\\
	\end{equation}
\end{lemma}

\begin{lemma} 
	According to \cite{yang2014quasi}, for $f \in  PW_b$ (Paley-Wiener spaces), we have 
	\begin{equation}
		\epsilon_{S,p}[f] \leq \Vert f\Vert_{PW_b}D^{\square b}_p(S)\\
	\end{equation}
	where 
	\begin{equation}
		\begin{aligned}
			D^{\square b}_p(S)^2 = &\pi^{-d} \int_{\beta \in \square b} |\Psi(\beta)|^2 d\beta-\frac{2(2\pi)^{-d}}{s}\sum_{j=1}^s \xi_j \int_{\beta \in \square b} \Psi(\beta)e^{i\beta^Tw_j} d\beta+ \frac{1}{s^2}\sum_{l=1}^s\sum_{j=1}^s sinc_b(w_l,w_j).
		\end{aligned}
	\end{equation}
	
	Suppose that $p(\cdot)$ is a probability density function. Let $\Psi(\cdot)$ be the characteristic function associated with $p(\cdot)$.
\end{lemma}
Following~\cite{yang2014quasi}, we can derive the following discrepancy measure that takes into account the weights:
\begin{equation}
	\scriptsize
	\begin{aligned}
		D^{\square b}_p(S, \Xi)^2 = &\pi^{-d} \int_{\beta \in \square b} |\Psi(\beta)|^2 d\beta
		-2(2\pi)^{-d}\sum_{j=1}^s \xi_j \int_{\beta \in \square b} \Psi(\beta)e^{i\beta^Tw_j} d\beta
		+ \sum_{l=1}^s\sum_{j=1}^s \xi_l\xi_j \sin c_b(w_l,w_j) \\
		=& \pi^{-d} \int_{\beta \in \square b} |\Psi(\beta)|^2 d\beta
		-2(2\pi)^{-d}\sum_{j=1}^s \zeta_j(x)\zeta_j(z) \int_{\beta \in \square b} \Psi(\beta)e^{i\beta^Tw_j} d\beta
		+ \sum_{l=1}^s\sum_{j=1}^s \zeta_l(x)\zeta_l(z)\zeta_j(z)\zeta_j(z) \sin c_b(w_l,w_j). \\
	\end{aligned}
\end{equation}
We note that if we fix the points, then optimizing just the weights is a optimization problem. The discrepancy measure can be written as
\begin{equation}
	\begin{aligned}
		D^{\square b}_p(S, \xi)^2 = &\pi^{-d} \int_{\beta \in \square b} |\Psi(\beta)|^2 d\beta -2v^T\xi +\xi^TH\xi \\
		=& \pi^{-d} \int_{\beta \in \square b} |\Psi(\beta)|^2 d\beta -2v^T(\zeta(x)\otimes\zeta(z)) 
		+ (\zeta(x)\otimes\zeta(z))^TH(\zeta(x)\otimes\zeta(z)), \\
	\end{aligned}
\end{equation}
where $\xi = (\xi_1,...,\xi_s)^T$, $\zeta(x)=(\zeta_1(x),...,\zeta_s(x))^T$, and $\xi=\zeta(x)\otimes\zeta(z)=(\zeta_1(x)\zeta_1(z),...,\zeta_s(x)\zeta_s(z))^T$. $\otimes$ is the Hadamard product. $v \in R^s$, $v_j = (2\pi)^{-d} \int_{\beta \in \square b} \Psi(\beta)e^{i\beta^Tw_j} d\beta$  and $H \in R^{s\times s}$, $H_{lj}=sinc_b(w_l,w_j)$.

Thus, the optimal weights can be found by solving the following convex optimization problem
\begin{equation}
	\begin{aligned}
		 &\arg\min_{\zeta(x),\zeta(z) \in R^s} -2v^T(\zeta(x)\otimes\zeta(z)) 
		+ (\zeta(x)\otimes\zeta(z))^TH(\zeta(x)\otimes\zeta(z))\\
		&\text{s.t.} \zeta(x) \leq 0, \zeta(z) \leq 0.
	\end{aligned}
\end{equation}
So obviously for a fixed sampling point, there is $D^{\square b}_p(S, \xi) \leq D^{\square b}_p(S)$. Then the upper bound of the integral estimation error of the objective function by WPFF (Weighted Positive Fixed Features) method is not greater than the upper bound of the integral estimation error of the objective function by PFF (Positive Fixed Features) method.


\end{proof}

\subsection{Examples generated by DiJiang-7B}

\begin{center}
	\begin{tcolorbox}[colback=gray!10,
		colframe=black,
		width=16cm,
		arc=1mm, auto outer arc,
		boxrule=0.8pt,
		]
		\textbf{Prompt:} There is a time when I had to make a difficult decision.
		\newline\newline
		\textbf{DiJiang-7B:} I was in a situation where I could either stay or leave. It was a very difficult time for me. I was very happy in my job. But I felt that I needed to move on. So I decided to leave my current job and start my own business. This was not an easy decision to take. There were many things that needed my attention. And I didn’t have enough time to do everything.
		\newline\newline
		\textbf{LLaMA2-7B:} I had to choose between my family and the people who were close with me. It was a hard decision, but it turned out that they weren’t really friends after all! I am not sure if you have ever been in this situation before or what your thoughts are on making such an important choice like choosing one over another person when there is no clear answer as far back then goes…but here we go anyway: “The best way forward would be for both parties involved (you included)to come together again soon so everyone can move past whatever happened.
	\end{tcolorbox}
\end{center}

\begin{center}
	\begin{tcolorbox}[colback=gray!10,
		colframe=black,
		width=16cm,
		arc=1mm, auto outer arc,
		boxrule=0.8pt,
		]
		\textbf{Prompt:} The sun goes down, and finally Gauss and Curie find time to relax and discuss after an exhausting day of work.
		\newline
		\newline
		\textbf{DiJiang-7B:} "So, what do you think about the newest research on the effects of the magnetic field on human brain?" Gau asks, taking a sip of his coffee. He's been working on a new project for the last few weeks, trying to find a way to use the Earth's magnetic fields to help people with disabilities. It'd be a huge breakthrough for science, but it'll take a lot of time and money to get it to work properly. Gaug is a bit of a dreamer, so he'd rather not think too much about it. Curry, on other hand, is more pragmatic. She'll be the first to tell you that the world doesn't need more scientists, it needs more people who can actually do something. But she'l always been a realist, too. So, when Gaus asks her opinion, she doesn t hesitate to give it, even though she knows it might not be what he wants to hear. "Well, I think it s a great idea, Gause. I mean, if we can find some way of using the earth's magnetic energy to improve people’s lives, that would be amazing. And I know you're working hard on it." Currey says, sounding a little bit tired. They both know that it ll take some time, though. The project is still in its early stages, after all. Still, Curreys smile is enough to make Gauc smile back. After all, they're both scientist. That means they can talk about science all day long, right?
		\newline
		\newline
		\textbf{LLaMA2-7B:} Gauss: So, how did you like the new job? I heard that it was a bit of an adjustment. You’re not used to working with people who are so… well-rounded and diverse in their interests! It must be quite different from your old life as one of Newton's apprentices at Cambridge University where all they cared about were numbers... or maybe just pi... but anyway - what do YOU think is more important for success these days – being able to calculate complex mathematical formulas quickly on paper using only pencil \& eraser OR having good social skills when dealing w/ other humans?? Hmmm.... let me know if u find out!! Bye now!!!
	\end{tcolorbox}
\end{center}

\end{document}